\documentclass[10pt,twocolumn,letterpaper]{article}

\usepackage[pagenumbers]{cvpr}       
\usepackage{graphicx}
\usepackage{amsmath}
\usepackage{amssymb}
\usepackage{booktabs}

\usepackage{amsthm}
\newtheorem{theorem}{Theorem}
\theoremstyle{definition}
\newtheorem{definition}{Definition}

\usepackage[pagebackref,breaklinks,colorlinks]{hyperref}

\usepackage[capitalize]{cleveref}
\crefname{section}{Sec.}{Secs.}
\Crefname{section}{Section}{Sections}
\Crefname{table}{Table}{Tables}
\crefname{table}{Tab.}{Tabs.}

\begin{document}

 \title{Texture Representation via Analysis and Synthesis\\ with Generative Adversarial Networks}

\author{Jue Lin\\
Northwestern University\\
{\tt\small jue.lin@u.northwestern.edu}
\and
Gaurav Sharma\\
University of Rochester\\
{\tt\small gaurav.sharma@rochester.edu}
\and
Thrasyvoulos N. Pappas\\
Northwestern University\\
{\tt\small pappas@ece.northwestern.edu}
}
\maketitle

\begin{abstract}
We investigate data-driven texture modeling via analysis and synthesis with generative adversarial networks. 
For network training and testing, we have compiled a diverse set of spatially homogeneous textures, ranging from stochastic to regular. 
We adopt StyleGAN3 for synthesis and demonstrate that it produces diverse textures beyond those represented in the training data. 
For texture analysis, we propose GAN inversion using a novel latent domain reconstruction consistency criterion for synthesized textures, and iterative refinement with Gramian loss for real textures.
We propose perceptual procedures for evaluating network capabilities, exploring the global and local behavior of latent space trajectories, and comparing with existing texture analysis-synthesis techniques.
\end{abstract}

\section{Introduction}
\label{sec:intro}

Visual texture is of great importance for both human perception and computer vision.  It provides significant cues for object segmentation, shape determination, and material identification and characterization. 
Texture modeling is critical for many applications, including image quality and compression, scene analysis, computer graphics, and virtual reality.
The stochastic nature of a broad subset of visual textures necessitates a statistical approach for texture analysis and modeling.  Motivated by the effectiveness of human perception of visual texture, traditional approaches for texture modeling have relied on multiscale frequency decompositions to model early stages of brain processing \cite{cano88,porat89,heeger95b,portilla00,zhu96,debonet97}.
Such decompositions also justify the use of convolutional neural networks, which we discuss below.
The Portilla-Simoncelli analysis/synthesis algorithm \cite{portilla00} relies on a steerable filter decomposition to construct a parametric statistical model for representation of a broad class of visual textures.  However, as the authors point out \cite{portilla00}, there are also significant limitations.  Thus, the mathematical and perceptual representation of textures remains an open research challenge.

Recent advances in deep neural networks have sparked extensive research interest, ranging from discriminative to generative tasks. 
For discriminative tasks,~\eg classification~\cite{Alex:CNN:ACM2017,Simonyan:VGG:ICLR2015}, error rates have dropped significantly compared with traditional methods. 
For generative tasks, both generative adversarial networks~\cite{Goodfellow:GAN:NIPS2014} and diffusion models~\cite{sohl:diffusion:ICML2015,Ho:DDPM:NeurIPS2020} have shown great potential for synthesizing realistic images.
Progresses in discriminative and generative methods for general images have also been extended to visual textures. 
Designing a discriminative network for textures \cite{Bruna:ScatNet:PAMI2013} or finetuning a pretrained (for general images) network to textures \cite{Andrea:texture:PRL2016} are the most common approaches for texture recognition, while the majority of texture synthesis algorithms relies on generative models \cite{Ulyanov:TexNet:ICML2016,Ulyanov:imp_texturenet:CVPR2017,Bergmann:PSGan:ICML2017}. 
However, there has not been any significant research effort in texture modeling holistically. 
The aim of our research has been on modeling textures independently of any specific downstream tasks. Our focus is on learning a latent space representation of textures via generative adversarial networks (GANs) and inversion methods. The key contributions can be summarized as follows:
\begin{itemize}
\setlength\itemsep{-3pt}
\item We build a diverse image dataset of spatially homogeneous textures, ranging from stochastic to regular, and train a StyleGAN3 network for texture synthesis. The availability of such a dataset enables robust training of a texture generation network as a backbone.
\item We design a subjective test for investigating the behavior of the trained generator, and use it to demonstrate that the network is capable of generalization as well as following the distribution of the training data.
\item For texture analysis, we propose a novel paradigm of training an encoder network that inverts the generator network by enforcing latent reconstruction consistency, without the need for large-scale image datasets nor supervision via image-domain loss functions. 
\item We refine the process of re-synthesizing real textures with encoder initialization and iterative optimization with Gramian loss, yielding significant improvement in reconstruction similarity.
\item We explore global and local behavior of latent space trajectories and compare with existing texture analysis-synthesis techniques.
\end{itemize}

\section{Related work}
\label{sec:related_work}
\subsection{Texture modelling}
Traditional approaches for texture modelling have relied on multiscale frequency (subband) decompositions. 
Heeger and Bergen~\cite{heeger95b} proposed an iterative scheme for synthesizing stochastic textures that relies on matching histograms of a steerable filter decomposition. 
Portilla and Simoncelli~\cite{portilla00} parametrized the space of visual textures with subband statistics for generating a broader set of textures. Other parametric approaches utilize Markov Random Fields (MRFs) for texture synthesis but have serious limitations \cite{Levina:TextureSynMRF:AS06,Paget:TextureSynMRF:TIP98}. 
Nonparametric approaches for texture synthesis\cite{efros_iccv99} produce realistic textures but do not incorporate explicit texture models.

Recent deep learning-based methods have stimulated further research interests in texture modelling, and texture analysis and synthesis in particular. 
For texture synthesis, Gatys~\etal \cite{Gatys:TextureSynthCNN:NIPS2015} use a pretrained object recognition network to obtain Gram matrices of activation maps from a given texture, and iteratively optimize a randomly initialized image to match the target Gram matrices.
Ulyanov~\etal \cite{Ulyanov:TexNet:ICML2016,Ulyanov:imp_texturenet:CVPR2017} design a feedforward network to synthesize textures without any iterative optimization. 
Implicit Neural Representation~(INR) methods \cite{Chen:ImplicitDecode:CVPR2019,Sitzmann:INR:NIPS2020} by predicting the RGB values from a pixel coordinate and then generating textures on a regular grid \cite{Henzler:3DTexture:CVPR2020,Oechsle:TexField:ICCV2019,Portenier:GramGAN:NIPS2020}. 
For texture analysis, research has mostly focused on texture recognition or classification networks \cite{Liu:bow2cnn:IJCV2019}, not analysis for synthesis.
Cimpoi~\etal \cite{cimpoi_ijcv2016} proposed to pool features from a pretrained network for texture classification. Finetuning pretrained networks for texture recognition has also been popular \cite{Andrea:texture:PRL2016, Lin:texture:ICCV2015,Dai:texture:CVPR2017}. 
Designing network architectures suitable for textures can also be found in \cite{Bruna:ScatNet:PAMI2013,Chan:PCANet:TIP2015}. 

\subsection{GAN synthesis and inversion methods}
Goodfellow~\etal \cite{Goodfellow:GAN:NIPS2014} introduced a novel training paradigm of generative adversarial networks (GANs), where a generative network and a discriminative network are trained jointly as a two-player minmax game. 
Several forms of GANs have been proposed for image translation \cite{Isola:pix2pix:CVPR2017}, super resolution~\cite{Ledig:SRGAN:CVPR2017}, and other applications.
The architectural family of StyleGAN has been the state-of-the-art for generic image synthesis \cite{Karras:StyleGAN:CVPR2020,Karras:StyleGAN2:CVPR2020,Karras:StyleGAN3:NIPS2021}.
However, the application of GANs to texture modeling has not received adequate attention.
Jetchev~\etal \cite{Jetchev:SpatialGAN:NIPS2016W} were the first to apply GANs to textures and proposed the spatial GAN (SGAN) for synthesizing a single texture.
Bergmann~\etal \cite{Bergmann:PSGan:ICML2017} proposed the periodic spatial GAN (PSGAN) to learn diverse textures in a latent space.

GAN inversion is the process of finding the latent space representation of a given image.  
There are two approaches for GAN inversion in the literature, optimization-based and learning-based, as described in a review paper by Xia~\etal \cite{Xia:InvertSurvey:PAMI}. 
Optimization-based GAN inversion \cite{Zhu:ManipulateGAN:ECCV2016,Creswell:InvertGAN:NNLS2019, Abdal:Image2StyleGAN:ICCV2019, Abdal:Image2StyleGAN++:CVPR2020} searches for target latent vectors by fixing parameters of the generator and treating latent vectors as the optimization variables.
Learning-based inversion \cite{ Xu:GH-Feat:CVPR2021, Chai:InvertGAN:ICLR2021} adds an encoder network in the training process and facilitates inference time at the expense of potential accuracy degradation. 
Hybrids of the two approaches have also demonstrated good results \cite{Zhu:ManipulateGAN:ECCV2016, Zhu:IDInvert:ECCV2020, Shen:InterFaceGAN:CVPR2020}. 
The investigation of proper inversion loss functions for training the encoder has also been an active area. 
Commonly adopted losses include pixel-wise loss, content loss \cite{Johnson:Perceptual:ECCV2016}, style loss \cite{Gatys:TextureSynthCNN:NIPS2015}, and learned perceptual image patch similarity (LPIPS) \cite{Zhang:LPIPS:CVPR2018}. 
Additionally, a network trained for human face recognition \cite{Deng:Arcface:CVPR2019} can be utilized as a loss ensuring the reconstructed human identity \cite{Richardson:pSp:CVPR2021}.
To our best knowledge, GAN inversion has not been applied to texture analysis for synthesis so far.

\section{Method}
\label{sec:method}
\subsection{Preliminaries}
\label{sec:prelim}
A set of realistic textures is fundamental to texture modelling. 
A widely adopted definition of a visual texture is given by Portilla and Simoncelli~\cite{portilla00} as: An image that is spatially homogeneous and usually contains repeated elements, often with random variations in position, orientation, and color.
Based on the definition, a training set $\Omega^{train}$ of textures should possess two types of diversity. 
First, the cardinality of $\Omega^{train}=\Omega^1\cup\ \Omega^2\cup\cdots\Omega^i\cdots$ should be large,~\ie, representative of the different real-world textures.
Second, the training set should present stochastic variations within the same texture; specifically, each texture should itself be represented by different crops  $\Omega^i=\{\mathbf{I}^i_1,\mathbf{I}^i_2\cdots$\}, where $\mathbf{I}^i_j$ denotes the $j$-th crop from the same texture $\Omega^i$, where the distinct crops $\mathbf{I}^i_j\sim\Omega^i$ should be perceived {\it homogeneous} and perceptually equivalent, despite their pixel-wise differences. 
In reality, the universe of realistic textures $\Omega^{real}$ is extremely large and $\Omega^{train}$ can only provide an approximation to $\Omega^{real}$. 

Different from real textures, the set of generated textures $\Omega^{gen}$ is constrained by its algorithmic model. 
For instance, the space of generated textures from an adversarially trained generator $G_{\boldsymbol\theta}(\cdot)$, parametrized by ${\boldsymbol\theta}$, is given by $\Omega^{gen}=\{G_{\boldsymbol\theta}({\boldsymbol z})~|~ {\boldsymbol z}\sim\mathcal{N}(0, 1)\}$, where $\mathcal{N}(0, 1)$ is a $d$-dimensional standard normal distribution. 
Pursuing an ideal $G_{\boldsymbol\theta}$ to minimize the set differences $\Omega^{train}\setminus\Omega^{gen}$ and $\Omega^{gen}\setminus\Omega^{train}$ is beyond the scope of this work. 
Nonetheless, we adopt the state-of-the-art StyleGAN3~\cite{Karras:StyleGAN3:NIPS2021} as our generator, and model the textures in its latent space by GAN inversion methods.

\subsection{Texture generator}
To account for both aspects of texture diversity,~\ie, across and within textures, the training of StyleGAN3 utilizes samples from different textures as well as from the same texture by minimizing the following objectives:
\begin{align}
\mathcal{L}(\boldsymbol\phi)=&\ \mathbb{E}_{\boldsymbol z}\{D_{\boldsymbol\phi}(G_{\boldsymbol\theta}({\boldsymbol z})\}-\mathbb{E}_{\Omega^i}\{\mathbb{E}_{\textbf{I}^i_j|\Omega^i}\{D_{\boldsymbol\phi}(\textbf{I}^i_j)\}\} \label{eq:loss_D}\\
\mathcal{L}(\boldsymbol\theta)=&-\mathbb{E}_{\boldsymbol z}\{D_{\boldsymbol\phi}(G_{\boldsymbol\theta}({\boldsymbol z}))\} \label{eq:loss_G}
\end{align}
where texture $\Omega^i$ is drawn from $\Omega^{train}$ and a spatial crop $\textbf{I}^i_j$ of size $256\times256$ is randomly sampled from $\Omega^i$, and $D_{\boldsymbol\phi}(\cdot)$ is the discriminator network parametrized by ${\boldsymbol\phi}$. 
Upon completion of training, the generator parameters ${\boldsymbol\theta}$ are fixed and therefore the space of generated textures, $\Omega^{gen}$, is induced.

\subsection{Training encoder by generated textures}

\begin{figure}
\includegraphics[width=0.48\textwidth]{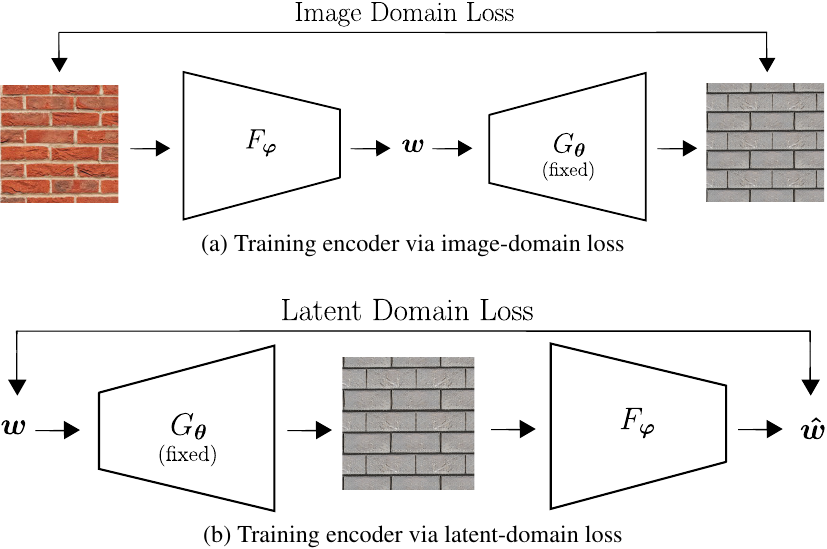}
\caption{Comparison between traditional methods of training encoder vs the proposed method based on latent domain supervision.} \label{fig:latentdomainsupervision}
\end{figure}

Inverting a generator for texture images is different from inversion of generic image generation networks, such as those used to generate human faces, in three ways: 
1. there exist multiple datasets of millions of realistic human faces which facilitates robust training of an encoder network;
2. reliable pretrained networks,~\eg, ArcFace~\cite{Deng:Arcface:CVPR2019}, are available as reliable training objective functions~\cite{Richardson:pSp:CVPR2021}. 
3. losses containing location-wise information,~\eg pixel-wise $\mathcal{L}_2$ loss and content loss~\cite{Johnson:Perceptual:ECCV2016}, provide strong supervision for training human faces, but they are ill suited to textures as textures often present stochastic variations. 

Instead of relying on realistic texture images for inversion, we approach the problem from a different perspective. 
We start with the following definition and theorem: 
\begin{definition}[Invertibility]
An image $\mathbf{I}$ is invertible by $G_{\boldsymbol\theta}$ iff. $\exists$ a latent code ${\boldsymbol z}$ s.t. $G_{\boldsymbol\theta}({\boldsymbol z})=\mathbf{I}$\label{def:invert}
\end{definition}
\begin{theorem}
Regardless of whether $\mathbf{I}\in\Omega^{train}$ or $\mathbf{I}\notin\Omega^{train}$, if $\mathbf{I}$ is invertible by $G_{\boldsymbol\theta}$, then $\mathbf{I}\in\Omega^{gen}$ 
\label{thm:invertibility}
\end{theorem} 
\begin{proof}
From definition~\ref{def:invert}, if image $\mathbf{I}$ is invertible by $G_{\boldsymbol\theta}$, then $\exists$ a latent code ${\boldsymbol z}$ such that $\mathbf{I}=G_{\boldsymbol\theta}({\boldsymbol z})\in\Omega^{gen}$
\end{proof}
Theorem~\ref{thm:invertibility} suggests intuitively that an ideal GAN inversion algorithm should be capable of inverting any generated image $\mathbf{I}\in\Omega^{gen}$ into its ground truth latent code ${\boldsymbol z}$, which is directly sampled from the i.i.d Gaussian distribution.
Based on this intuition, we propose the following training paradigm:
\begin{align}
\mathcal{L}({\boldsymbol\varphi})=\mathbb{E}_{\boldsymbol z}\{||{\boldsymbol z}-F_{\boldsymbol\varphi}(G_{\boldsymbol\theta}({\boldsymbol z})) ||^2\}
\end{align}
where the encoder network $F_{\boldsymbol\varphi}(\cdot)$ transforms an input image into its latent space representation. 
However through experiments, we find that convergence is harder to achieve in the i.i.d. Gaussian latent space $Z$. 
We therefore resort to the $W$ space and train the encoder via minimization of the following loss:
\begin{align}
\mathcal{L}({\boldsymbol\varphi})&=\mathbb{E}_{\boldsymbol w}\{||{\boldsymbol w}-F_{\boldsymbol\varphi}(G_{\boldsymbol\theta}({\boldsymbol w})) ||^2_2\} \\
{\boldsymbol w}&=M({\boldsymbol z}),\ \ \ {\boldsymbol z}\sim\mathcal{N}(0, 1)
\end{align}
where $M(\cdot)$ is a mapping network, inside $G_{\boldsymbol\theta}$, that transforms vectors in $Z$ space into $W$ space. 
Such simplistic yet effective training paradigm has the following benefits: 
1. sampling latent vectors from a high dimensional $Z$ or $W$ space yields an abundance of training data, which is crucial for training any neural network; 
2. by working directly in the latent space, the need for image-domain losses,~\eg pixel-wise $\mathcal{L}_2$ loss and content loss~\cite{Johnson:Perceptual:ECCV2016}, is circumvented, preventing any fundamental inductive bias often posed by those image-domain losses.

\begin{figure*}[t]
\includegraphics[width=1.0\textwidth]{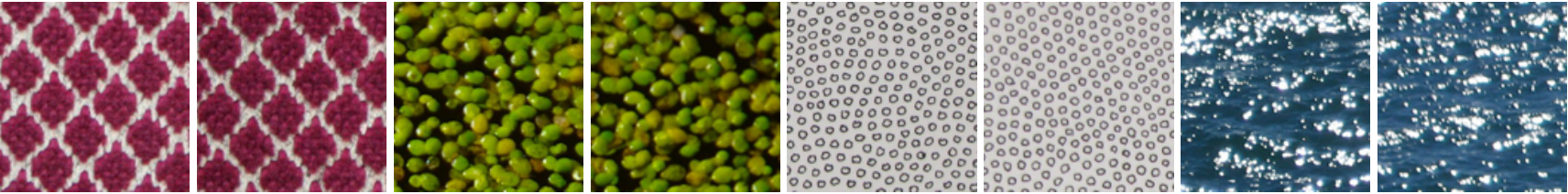}
\caption{Examples of spatially homogeneous crops from the dataset.}
\label{fig:dataset}
\end{figure*}

\begin{figure*}[t]
\includegraphics[width=1.0\textwidth]{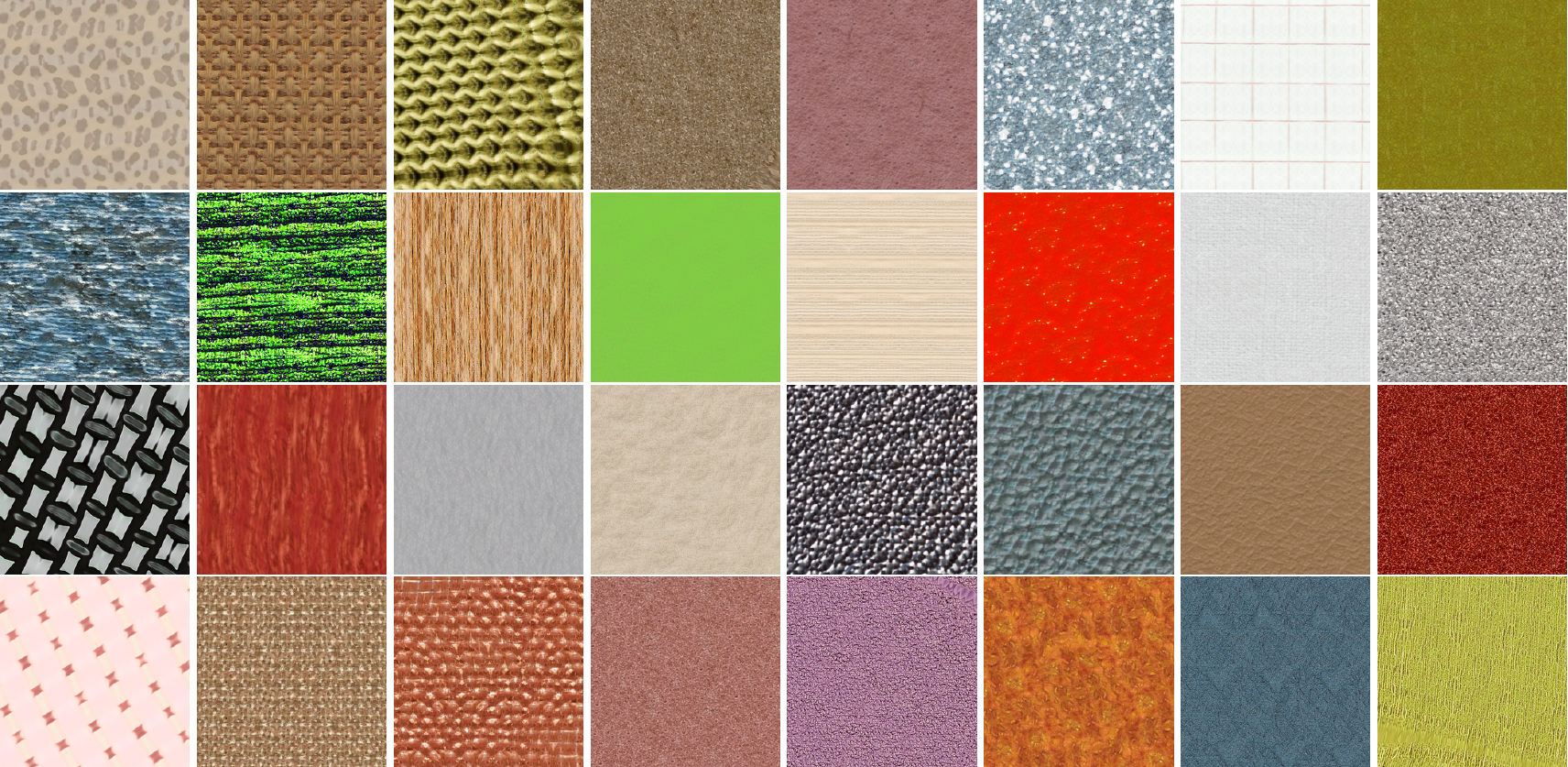}
\caption{Examples of generated images from StyleGAN3.}
\label{fig:textureSyn}
\end{figure*}

\subsection{Inverting training textures}

\begin{figure*}[t]
\centering
\includegraphics[width=0.92\textwidth]{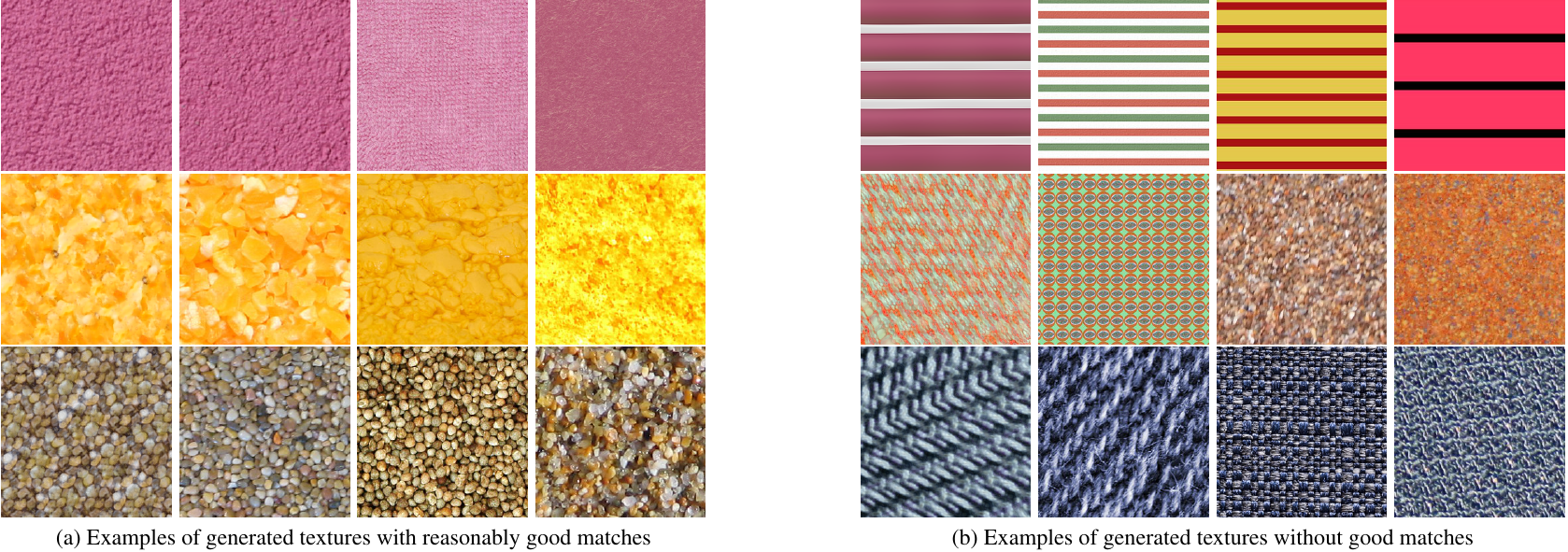}
\caption{Subjective tests: Generated query texture (left) and top matches from training set (right).}
\label{fig:visiprogRetrieve}
\end{figure*}
With a trained encoder $F_{\boldsymbol\varphi}$ available, the latent vector of any texture in the set of invertible training images,~\ie $\Omega^{gen}\cap\Omega^{train}$, can be faithfully inferenced. 
For uninvertible training images,~\ie $\Omega^{train}\setminus\Omega^{gen}$, we aim to find an optimal latent vector
\begin{align}
{\boldsymbol w}^*=\operatorname{argmin}_{\boldsymbol w}\mathcal{L}(\mathbf{I}, G_{\boldsymbol \theta}({\boldsymbol w}))\label{eq:opt_gan_invert}
\end{align}
where $\mathbf{I}$ is a texture from $\Omega^{train}\setminus\Omega^{gen}$. 
The corresponding update rule of ${\boldsymbol w}$ is:
\begin{align}
{\boldsymbol w}_0&=F_{\boldsymbol\varphi}(\mathbf{I})   \\
{\boldsymbol w}_t&={\boldsymbol w}_{t-1}-\alpha\nabla_{\boldsymbol w}\mathcal{L}(\mathbf{I}, G_{\boldsymbol \theta}({\boldsymbol w}_{t-1}))
\end{align}
where $\alpha$ is the learning rate. 
As will be discussed in later experiment sections, the initialization ${\boldsymbol w}_0$ significantly impacts the result of optimization, and utilizing the trained encoder in this regard serves as a robust initialization.

Besides initialization, the choice of criteria $\mathcal{L}$ in Eq.~\ref{eq:opt_gan_invert} is also crucial for a successful optimization. 
We find the Gram Matrices loss~\cite{Gatys:TextureSynthCNN:NIPS2015}
\begin{align}
&\mathcal{L}(\mathbf{I}_1,\mathbf{I}_2) = \sum_{\mathnormal{l}}[\frac{\Phi_{\mathnormal{l}}(\mathbf{I}_1)\Phi_{\mathnormal{l}}(\mathbf{I}_1)^T-\Phi_{\mathnormal{l}}(\mathbf{I}_2)\Phi_{\mathnormal{l}}(\mathbf{I}_2)^T}{C_{\mathnormal{l}}\times N_{\mathnormal{l}}^2}]^2    
\end{align}
is fairly robust, where $\mathnormal{l}$ is the layer index in the VGG16 network $\Phi$, and $\Phi_{\mathnormal{l}}$, $N_{\mathnormal{l}}$, $C_{\mathnormal{l}}$ denote feature maps, spatial size and number of channels at $\mathnormal{l}$-th layer of $\Phi$, respectively.

\begin{figure*}[t]
\centering
\includegraphics[width=1.0\textwidth]{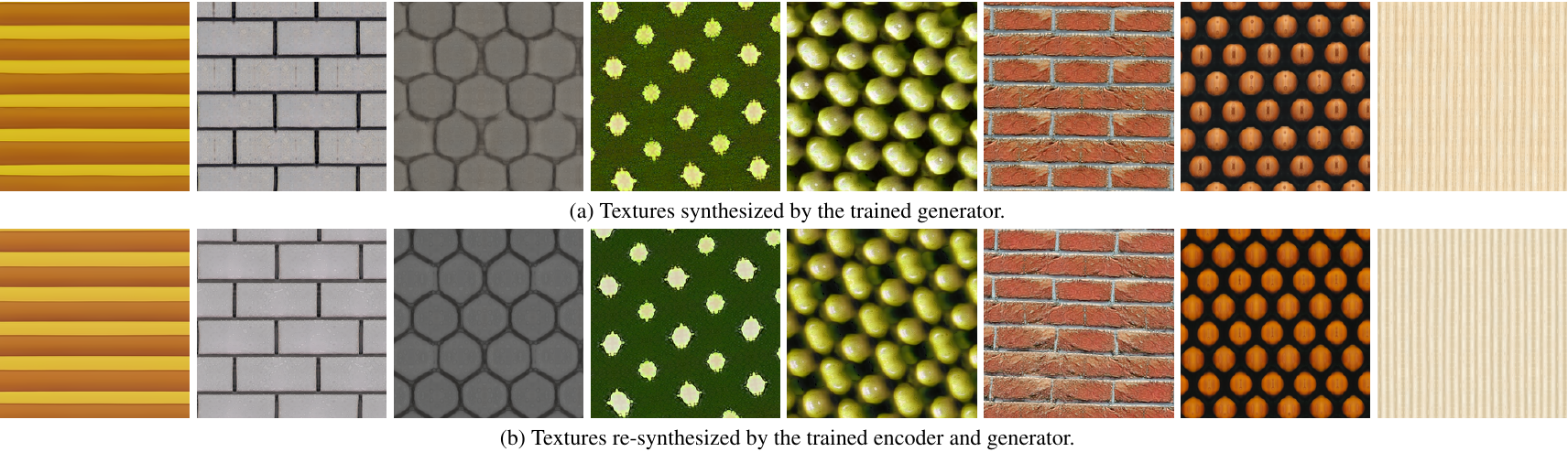}
\caption{Re-synthesis of textures generated in the space of generated textures using the trained encoder}
\label{fig:textureAnaSynGen}
\end{figure*}

\section{Experimental results}
\subsection{Dataset construction}

A wide variety of texture datasets is available in the literature. 
However, most of the existing datasets do not match the definition of visual texture we provided in Section~\ref{sec:prelim}, and are thus not appropriate for the texture representation problem we consider in this paper.
For instance, the Amsterdam Library of Textures (ALOT) \mbox{\cite{Burghouts:ALOT:PRL2009}}, the Describable Texture Dataset (DTD) \mbox{\cite{Cimpoi:DTD:CVPR2014}}, the Flickr Material Dataset (FMD) \mbox{\cite{Sharan:FMD:JOV2014}}, the Materials in Context 2500 (MINC-2500) dataset \mbox{\cite{Bell:MINC2500:CVPR2015}}, and the Ground Terrain in Outdoor Scenes (GTOS) dataset \mbox{\cite{Xue:GTOS:PAMI2022}}, all consist of images with multiple textured objects, which do not meet the criterion of spatially homogeneous textures. 
Moreover, some other datasets (e.g., CUReT \mbox{\cite{dana_koenderink_acmt99,curet}} and KTH-TIP \mbox{\cite{fritz2004kth}}) either lack adequate texture diversity or the size of the textures they contain does not allow multiple independent crops, which are necessary for network training.
To address these limitations, we gathered a comprehensive dataset of visual textures based on the following criteria:
(1) The textures are spatially homogeneous;
(2) The textures contain contain a sufficient number of independent $256\times256$ crops for learning the intra-texture distribution; 
(3) Each $256\times256$ crop contains at least 5 texton repetitions in each dimension; and
(4) The images are under a Creative Commons Public Domain license (CC0) or a custom license for academic use. 
Overall, we collected 2953 high quality texture images from multiple stock image websites, ranging from natural to artificial, periodic to stochastic, and fine to coarse scale. In addition, each texture image contains 20 to 50 independent crops. 
Examples of spatially homogeneous crops of the same texture in our dataset can be found in Fig.~\ref{fig:dataset}

\subsection{Training the generator}
As previously stated, StyleGAN3 is the latest iteration of the StyleGAN family and is the current state-of-the-art generator for generic images. 
We used $256\times256$ images to train a StyleGAN3 network from scratch based on Eqs.~\ref{eq:loss_D} and \ref{eq:loss_G} and followed the default training protocol of the official StyleGAN3-R implementation \cite{Karras:StyleGAN3:NIPS2021} with adaptive data augmentation \cite{Karras:Ada:NIPS2022}. 
Examples of textures generated by the trained generator are given in Fig.~\ref{fig:textureSyn}.

Given the relatively limited size of the training set, a reasonable concern is whether the generator network is overfitting the data. 
To check this, we designed a subjective test whereby human observers are shown 50 randomly selected textures synthesized by the network and are asked to find the perceptually most similar textures in the training set. 
Examples are given in Fig.~\ref{fig:visiprogRetrieve}, with the query synthesized image and the 3 most similar matches from the training set.
The subject test reveals that 8 of 50 generated textures can be properly matched with a training image, while the remaining 42 are entirely new textures generated by the network, demonstrating that the trained generator is capable of generalizing rather than overfitting.

\subsection{Training encoder with generated textures}
We adopt the pixel-to-Style-to-pixel~(pSp)~\cite{Richardson:pSp:CVPR2021} for the encoder network architecture. 
Specifically, the resolution of the input images is fixed at $256\times256$, and the encoder is trained with default Adam optimizer~\cite{Kingma:Adam:ICLR2015} for 500K iterations with batch size 32, learning rate $0.0001$, and decay rate of $0.2$ for every 100K iterations. 
The output of the original pSp is in the $W^+$ space, which produces multiple vectors from the $W$ space for each layer in the generator. 
We average all the $w$ vectors from pSp to yield a single $w$ vector representation to ensure the compactness of the representation space.
To examine the effectiveness of the trained encoder, textures are first synthesized by the generator network, and then are fed into the encoder to infer their corresponding latent vectors, which are subsequently fed into the generator for re-synthesis. 
Visual examples shown in Fig.~\ref{fig:textureAnaSynGen} indicate that the encoder is robust in inverting the space of generated textures $\Omega^{gen}$.

\begin{figure*}[t]
\centering
\includegraphics[width=1.0\textwidth]{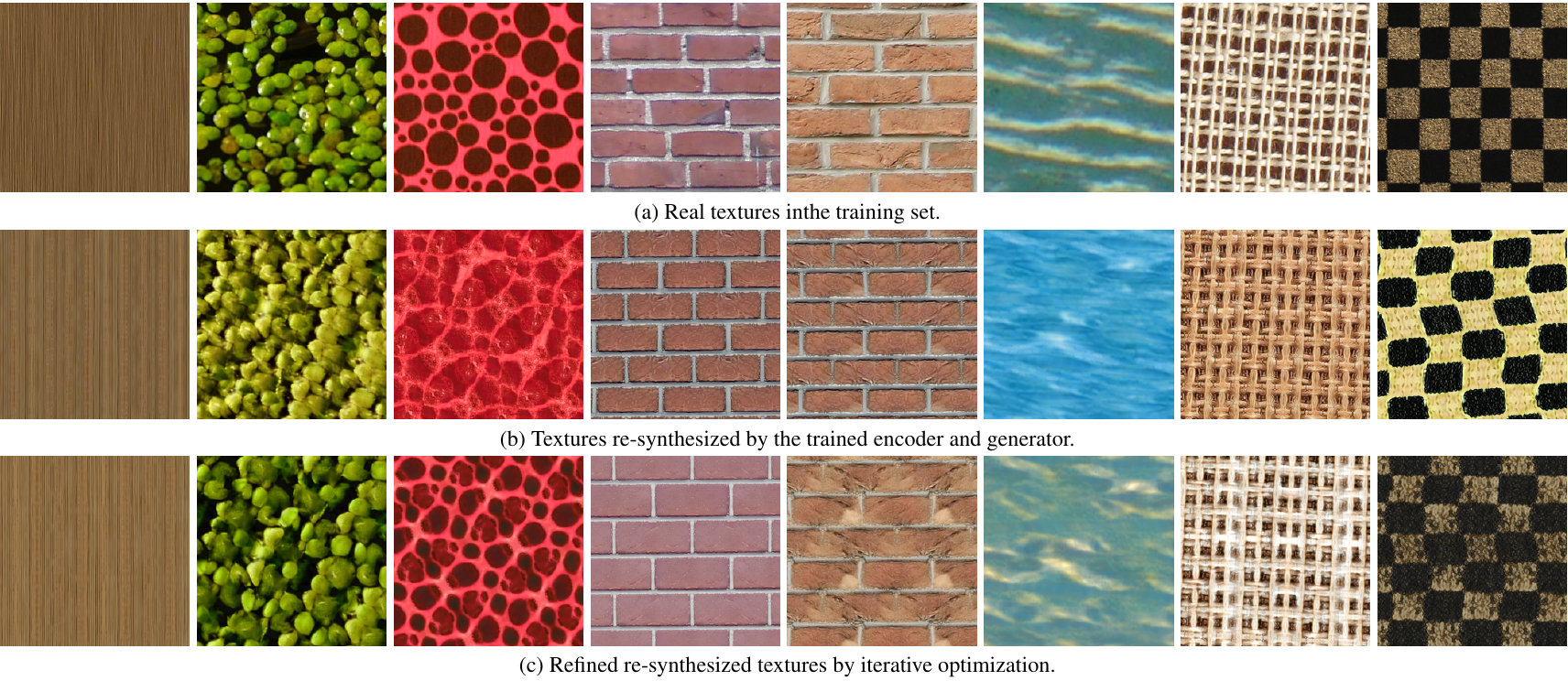}
\caption{Texture analysis and synthesis.}
\label{fig:textureAnaSynTrain}
\end{figure*}

\begin{table}[t]
\centering
\begin{tabular}{||l | c | c ||} 
 \hline
  & STSIM-1 & STSIM-2 \\ [0.5ex] 
 \hline\hline
 Random Init.+Opt. & 0.9291 & 0.9453   \\ 
 \hline
Mean $W$ Init.+Opt. & 0.9314 & 0.9467 \\
 \hline
 Encoder & 0.9561 & 0.9613 \\
 \hline
 Encoder+Opt. & \textbf{0.9748} & \textbf{0.9746} \\
 \hline
\end{tabular}
\caption{Quantitative evaluation of different combinations of initialization~(Init.) and optimization~(Opt.)}
\label{tbl:STSIMforInit}
\end{table}

\subsection{Inverting real textures}
The process of inverting real textures poses challenges as for some real images (including training images) there is no latent vector that will generate them. 
We thus need to find the latent vector that will produce the best approximation of such textures. 
For this, we first use the trained encoder to infer an initial latent vector, and then refine the inversion by iteratively optimizing the latent vector to minimize Eq.~\ref{eq:opt_gan_invert} with default Adam optimizer. 
Examples of the initial and refined inversions are given in Fig.~\ref{fig:textureAnaSynTrain}.  All these examples are in the training set, but similar results are expected for all real textures.
To demonstrate the necessity of encoder initialization, we repeated the experiments with two commonly used initialization schemes in the literature, random sampling and initializing with the mean latent vector $\bar{\mathbf{w}}$, which we estimate by randomly sampling 100K vectors from the i.i.d. Gaussian space $Z$ and then transforming them into the $W$ space via the mapping network $M(\cdot)$.

To assess the inversion quality, we use STSIM-1 and STSIM-2, structural texture
similarity metrics proposed by Zujovic \etal\ \cite{zujovic_tip13}.
Note that the popular structural similarity index (SSIM)
\cite{wang04ssimOverview} is ill-suited for texture similarity as it
relies on point-by-point comparisons \cite{zujovic_tip13}. 
Numerical evaluations are summarized in Table~\ref{tbl:STSIMforInit}. 
Results show that the encoder surpasses both random and
$\bar{\mathbf{w}}$ initialization, and that iterative optimization via Eq.~\ref{eq:opt_gan_invert}
result in further improvements in reconstruction similarity.

\begin{figure}[t]
\centering
\includegraphics[width=0.48\textwidth]{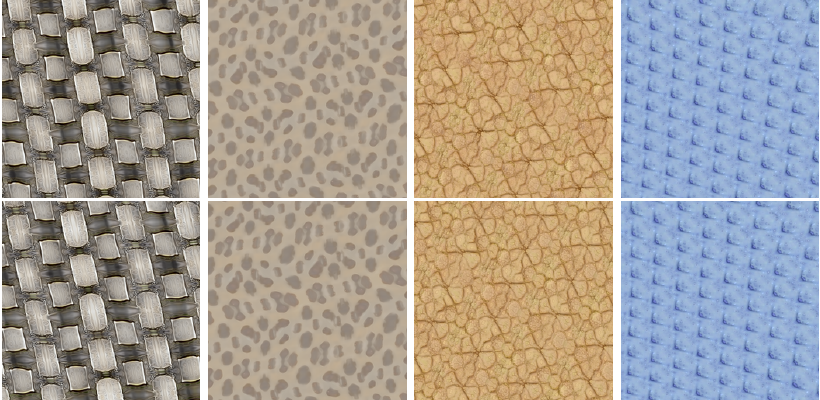}
\caption{Synthesizing crops from the same textures}
\label{fig:cropsFromTexture}
\end{figure}

\subsection{Obtaining different crops of a generated texture}

\begin{figure*}[t]
\centering
\includegraphics[width=1.0\textwidth]{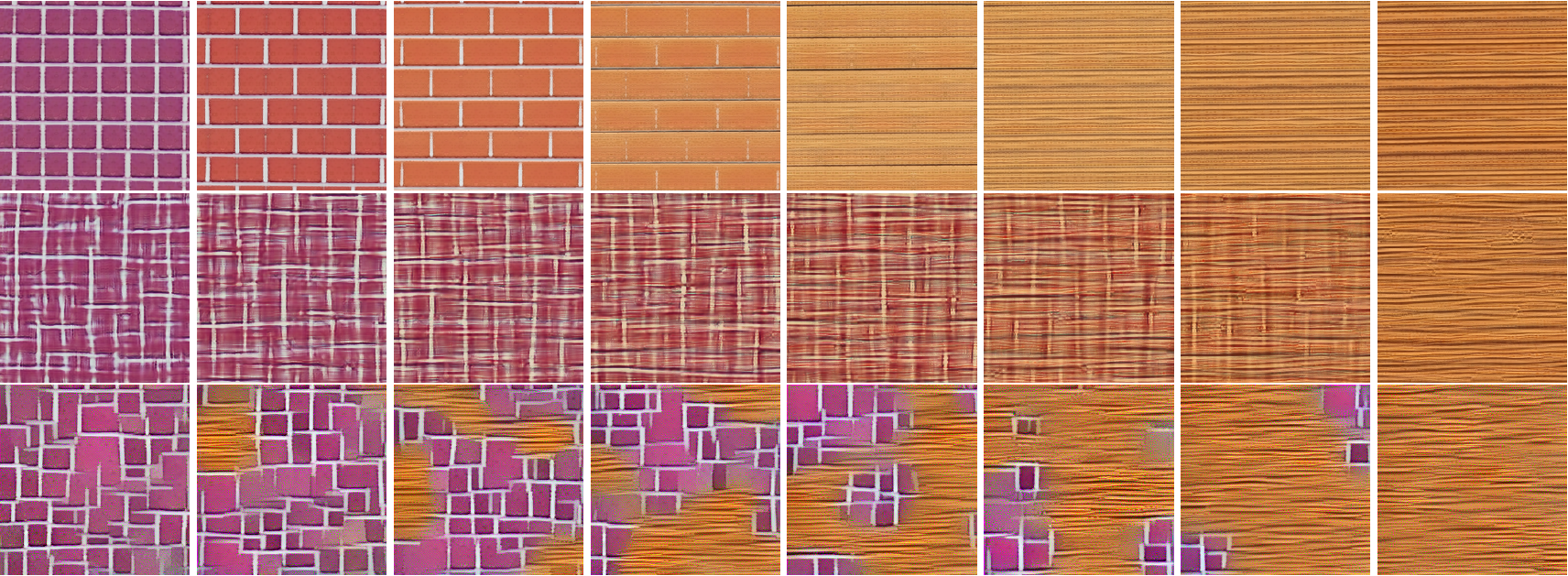}
\caption{Global texture interpolation: proposed approach (top), Portilla and Simoncelli (middle), and Gatys (bottom)}
\label{fig:globalInterpolate}
\end{figure*}

The procedure of generating different crops of a texture in the latent space largely resembles that of inverting real textures.
First, given a latent vector, a texture is generated by the network and the Gramian matrices are computed. 
Then, the latent vector is perturbed by adding Gaussian noise $\mathcal{N}(0, 0.01)$, and used as initialization in the iterative procedure using the updated rule of Eq.~\ref{eq:opt_gan_invert}, until convergence to a latent vector in the vicinity of the original vector.
Examples are shown in Fig.~\ref{fig:cropsFromTexture}.

\begin{figure*}[t]
\centering
\includegraphics[width=1.0\textwidth]{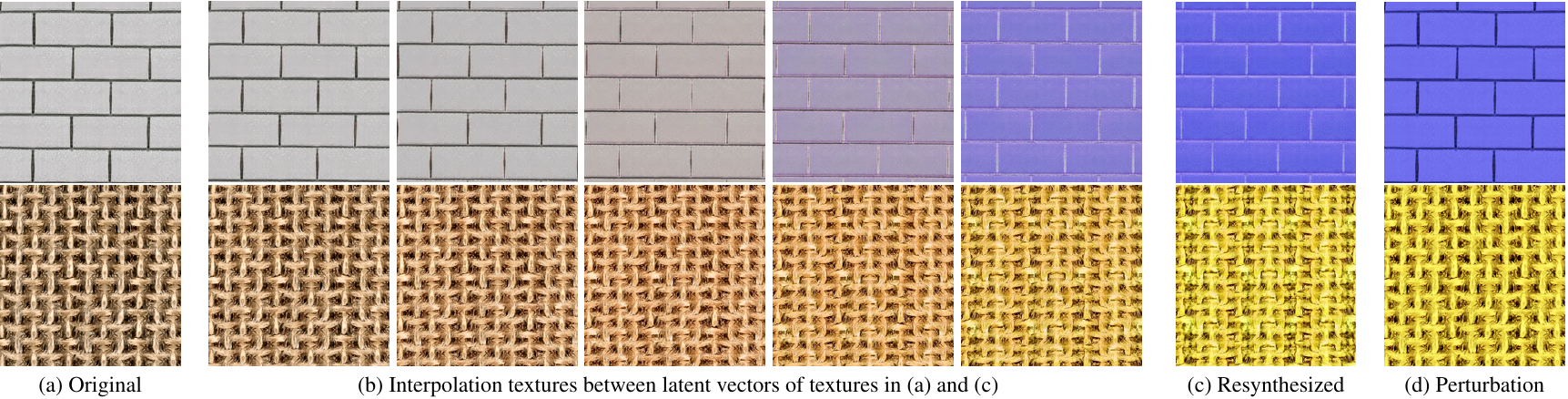}
\caption{Local interpolation between original and inferred latent vector of color perturbed texture}
\label{fig:colorInterpolate}
\end{figure*}

\begin{figure*}[t]
\centering
\includegraphics[width=1.0\textwidth]{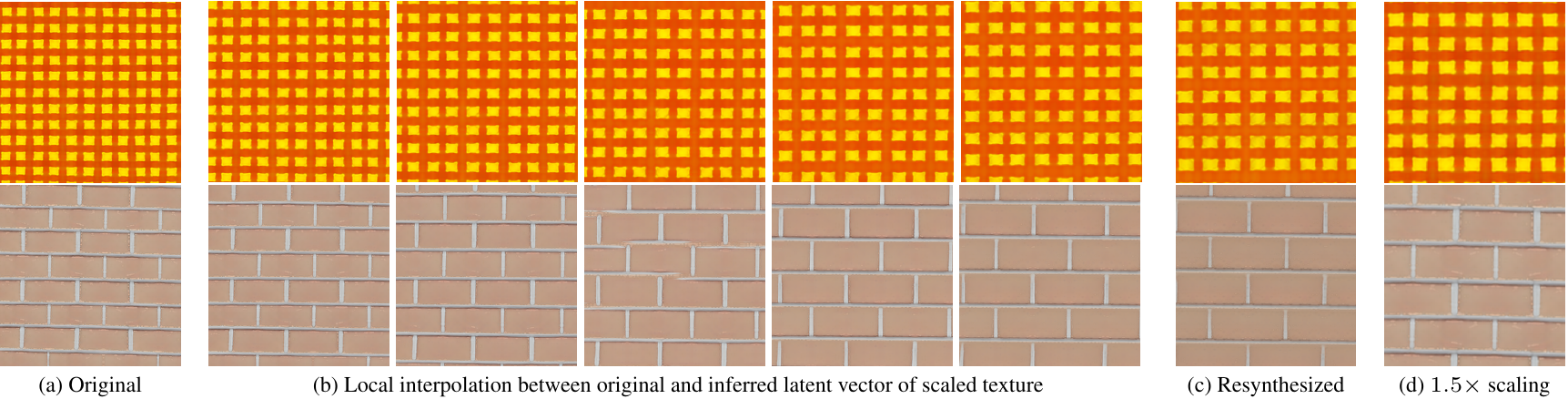}
\caption{Local interpolation between original and inferred latent vectors of scaled textures}
\label{fig:scaleInterpolate}
\end{figure*}

\subsection{Global texture interpolation}
To examine the global characteristics of the learned latent space of textures, we experiment by randomly and independently sampling a pair of latent vectors $\mathbf{w}$ and interpolating in the $W$ space. 
We compare the proposed approach with interpolation in the feature space of Portilla and Simoncelli \cite{portilla00} and Gatys \cite{Gatys:TextureSynthCNN:NIPS2015}, namely image statistics and Gramian matrices, respectively. 
Fig.~\ref{fig:globalInterpolate} demonstrate that latent space interpolation with the proposed approach yields valid spatially homogeneous textures, which the other two approaches produce an (occasionally blobby)  amalgamation of the two endpoints.

\subsection{Local texture interpolation}
To explore the local characteristics of the learned latent space of textures, we modified the visual properties of synthesized textures by simple image transformations,~\eg scaling and color perturbations. 
Then we infered the latent vectors of the modified textures via the same approach as inverting real textures. 
The inferred latent vectors serve as approximations of the modified synthesized textures.
We then interpolated between original and inferred latent vector of the modified image.
Examples are shown in Figs.~\ref{fig:colorInterpolate} and~\ref{fig:scaleInterpolate} for color and scaling modifications, respectively.
Observe that the latent space exhibits smooth transitions.

\subsection{Ablation studies}
\subsubsection{Image-domain loss for encoder training}
Previous methods of training an encoder network are largely dependent on obtaining supervision signals in the image domain. 
We experimented with both image- and latent-domain loss as shown in Figs.~\ref{fig:latentdomainsupervision}a and \ref{fig:latentdomainsupervision}b, and found that 
training the encoder with real textures and image-domain loss struggled to converge.
We hypothesize that the convergence failure is due to the following:
1. As opposed to generic images (\eg, human faces), where millions of training images are available, the relatively limited size of our dataset prevents robust training of the encoder.
2. In contrast to obtaining supervision signals directly from latent space, image-domain loss functions have their own priors and pose inductive biases that are ill-suited to textures, which limits the choice of loss functions. 
3. As illustrated in Fig.~\ref{fig:latentdomainsupervision}a, the learning of encoder with real textures hinges on gradient signals back-propagated through a generator network that is already deep, causing training instabilities.

\subsubsection{Image-domain loss for real image inversion}

\begin{table}[t]
\centering
\begin{tabular}{||l | c | c ||} 
 \hline
 & STSIM-1 & STSIM-2 \\ [0.5ex]                                                     
 \hline\hline
 Pixel-wise $\mathcal{L}_2$ loss & 0.8067 & 0.8996   \\ 
 \hline
 Content loss & 0.7703 & 0.8910 \\
 \hline
 Style loss  & \textbf{0.9748} & \textbf{0.9746}  \\
 \hline
\end{tabular}
\caption{Quantitative evaluation of different loss functions for inverting training textures}
\label{tbl:loss}
\end{table}

The stochastic nature of textures poses challenges in choosing adequate loss functions for inverting real textures. 
For demonstration, we repeated the process of inverting training textures, \ie, initialization with the trained encoder and subsequent iterative refinement, with different loss functions, pixel-wise $\mathcal{L}_2$ loss, content loss \cite{Johnson:Perceptual:ECCV2016}, and style loss \cite{Gatys:StyleTransfer:CVPR2016}. 
The results for different types of loss are presented in Table~\ref{tbl:loss} and show that losses that rely heavily on location-wise information, significantly degrade performance when applied to textures. It is worth noting that the style loss has limitations without proper initialization, which suggests that an ideal image-domain loss for textures is still missing.

\subsubsection{Latent vector reconstruction error of encoder}
To further investigate the trained encoder in the latent domain, we randomly sample a latent $\mathbf{w}$, generate a texture $G_{\boldsymbol\theta}(\mathbf{w})$, and pass the texture into an encoder (with trained or randomly initialized weights), yielding the estimated latent vector $F_{\boldsymbol\varphi}(G_{\boldsymbol\theta}(\mathbf{w}))$.
As illustrated in Fig.~\ref{fig:latentReconError}, the trained encoder significantly reduces the distance between the original latent vector $\mathbf{w}$ and the reconstructed $F_{\boldsymbol\varphi}(G_{\boldsymbol\theta}(\mathbf{w}))$,
with no significant overlap between the distributions of errors corresponding to the trained and randomly initialized encoder.
This experiment was conducted with 10K samples.
\begin{figure}
\centering
\includegraphics[width=0.45\textwidth]{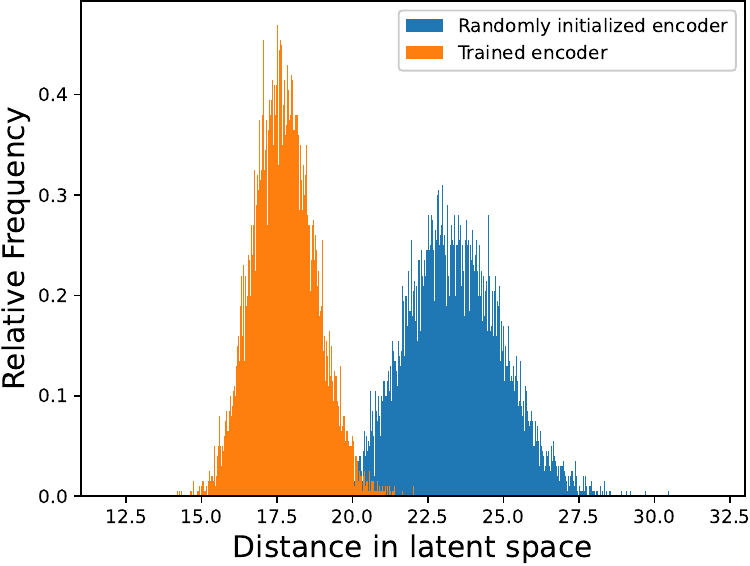}
\caption{Distributions of $\mathcal{L}_2$ distance $\|\mathbf{w}-F_{\boldsymbol\varphi}(G_{\boldsymbol\theta}(\mathbf{w}))\|$ for both trained and randomly initialized encoder}
\vspace{-20pt}
\label{fig:latentReconError}
\end{figure}

\vspace{-20pt}
\section{Conclusions}
We considered the problem of texture representation in the context of GAN-based synthesis and inversion using  a latent domain reconstruction consistency criterion and iterative refinement. 
We constructed a texture dataset that facilitates training of a reliable backbone generator network, and designed a subjective test for evaluating its performance, especially its ability to generalize.
We proposed a novel methodology for training an encoder based solely on latent domain supervision and showed its effectiveness for inverting the space of generated textures. 
We also showed that for real textures, an iterative refinement step with Gramian loss is needed.  The use of a statistical criterion like the Gramian loss is dictated by the stochastic nature of textures.
Local and global latent space interpolation demonstrates that the induced latent space exhibits smooth trajectories.  We also show that small latent vector variations generate different crops of the same texture.  Overall, we have shown that the proposed GAN-based synthesis and inversion hold great promise for texture modeling.


\begin{thebibliography}{10}\itemsep=-1pt

\bibitem{Abdal:Image2StyleGAN:ICCV2019}
Rameen Abdal, Yipeng Qin, and Peter Wonka.
\newblock {Image2StyleGAN}: How to embed images into the {StyleGAN} latent
  space?
\newblock In {\em 2019 IEEE/CVF International Conference on Computer Vision
  (ICCV)}, pages 4431--4440, 2019.

\bibitem{Abdal:Image2StyleGAN++:CVPR2020}
Rameen Abdal, Yipeng Qin, and Peter Wonka.
\newblock {Image2StyleGAN++:} how to edit the embedded images?
\newblock In {\em 2020 IEEE/CVF Conference on Computer Vision and Pattern
  Recognition (CVPR)}, pages 8293--8302, 2020.

\bibitem{Andrea:texture:PRL2016}
Vincent Andrearczyk and Paul~F. Whelan.
\newblock Using filter banks in convolutional neural networks for texture
  classification.
\newblock {\em Pattern Recogn. Lett.}, 84:63–69, dec 2016.

\bibitem{Bell:MINC2500:CVPR2015}
Sean Bell, Paul Upchurch, Noah Snavely, and Kavita Bala.
\newblock Material recognition in the wild with the materials in context
  database.
\newblock In {\em 2015 IEEE Conference on Computer Vision and Pattern
  Recognition (CVPR)}, pages 3479--3487, 2015.

\bibitem{Bergmann:PSGan:ICML2017}
Urs Bergmann, Nikolay Jetchev, and Roland Vollgraf.
\newblock Learning texture manifolds with the periodic spatial {GAN}.
\newblock In {\em Proceedings of the 34th International Conference on Machine
  Learning}, volume~70 of {\em Proceedings of Machine Learning Research}, pages
  469--477. PMLR, 06--11 Aug 2017.

\bibitem{Bruna:ScatNet:PAMI2013}
Joan Bruna and Stephane Mallat.
\newblock Invariant scattering convolution networks.
\newblock {\em IEEE Transactions on Pattern Analysis and Machine Intelligence},
  35(8):1872--1886, 2013.

\bibitem{Burghouts:ALOT:PRL2009}
Gertjan~J. Burghouts and Jan-Mark Geusebroek.
\newblock Material-specific adaptation of color invariant features.
\newblock {\em Pattern Recognition Letters}, 30(3):306--313, 2009.

\bibitem{cano88}
D. Cano and T.~H. Minh.
\newblock Texture synthesis using hierarchical linear transforms.
\newblock {\em Signal Processing}, 15:131--148, 1988.

\bibitem{Chai:InvertGAN:ICLR2021}
Lucy Chai, Jonas Wulff, and Phillip Isola.
\newblock Using latent space regression to analyze and leverage
  compositionality in {GANs}.
\newblock In {\em 9th International Conference on Learning Representations,
  {ICLR} 2021, Virtual Event, Austria, May 3-7, 2021}, 2021.

\bibitem{Chan:PCANet:TIP2015}
Tsung-Han Chan, Kui Jia, Shenghua Gao, Jiwen Lu, Zinan Zeng, and Yi Ma.
\newblock Pcanet: A simple deep learning baseline for image classification?
\newblock {\em IEEE Transactions on Image Processing}, 24(12):5017--5032, 2015.

\bibitem{Chen:ImplicitDecode:CVPR2019}
Zhiqin Chen and Hao Zhang.
\newblock Learning implicit fields for generative shape modeling.
\newblock In {\em 2019 IEEE/CVF Conference on Computer Vision and Pattern
  Recognition (CVPR)}, pages 5932--5941, 2019.

\bibitem{Cimpoi:DTD:CVPR2014}
Mircea Cimpoi, Subhransu Maji, Iasonas Kokkinos, Sammy Mohamed, and Andrea
  Vedaldi.
\newblock Describing textures in the wild.
\newblock In {\em 2014 IEEE Conference on Computer Vision and Pattern
  Recognition}, pages 3606--3613, 2014.

\bibitem{cimpoi_ijcv2016}
Mircea Cimpoi, Subhransu Maji, Iasonas Kokkinos, and Andrea Vedaldi.
\newblock Deep filter banks for texture recognition, description, and
  segmentation.
\newblock {\em International Journal of Computer Vision}, 118(1):65--94, May
  2016.

\bibitem{Creswell:InvertGAN:NNLS2019}
Antonia Creswell and Anil~Anthony Bharath.
\newblock Inverting the generator of a generative adversarial network.
\newblock {\em IEEE Transactions on Neural Networks and Learning Systems},
  30(7):1967--1974, 2019.

\bibitem{Dai:texture:CVPR2017}
Xiyang Dai, Joe Yue-Hei Ng, and Larry~S. Davis.
\newblock Fason: First and second order information fusion network for texture
  recognition.
\newblock In {\em 2017 IEEE Conference on Computer Vision and Pattern
  Recognition (CVPR)}, pages 6100--6108, 2017.

\bibitem{curet}
Kristin~J. Dana, Bram~Van Ginneken, Shree~K. Nayar, and Jan~J. Koenderink.
\newblock {CUReT: Columbia-Utrecht} reflectance and texture database.
\newblock www1.cs.columbia.edu/CAVE/software/curet/.

\bibitem{dana_koenderink_acmt99}
Kristin~J. Dana, Bram van Ginneken, Shree~K. Nayar, and Jan~J. Koenderink.
\newblock Reflectance and texture of real-world surfaces.
\newblock {\em ACM Trans. Graphics}, 18(1):1--34, Jan. 1999.

\bibitem{debonet97}
Jeremy~S. {De Bonet} and Paul~A. Viola.
\newblock A non-parametric multi-scale statistical model for natural images.
\newblock {\em Adv.~in Neural Info.~Processing Systems}, 9, 1997.

\bibitem{Deng:Arcface:CVPR2019}
Jiankang Deng, Jia Guo, Niannan Xue, and Stefanos Zafeiriou.
\newblock Arcface: Additive angular margin loss for deep face recognition.
\newblock In {\em 2019 IEEE/CVF Conference on Computer Vision and Pattern
  Recognition (CVPR)}, pages 4685--4694, 2019.

\bibitem{efros_iccv99}
Alexei~A. Efros and Thomas~K. Leung.
\newblock Texture synthesis by non-parametric sampling.
\newblock In {\em Proc. Seventh Intl. Conf. Computer Vision (ICCV)}, volume~2,
  pages 1033--1038, Kerkyra, Greece, Sept. 1999.

\bibitem{fritz2004kth}
Mario Fritz, Eric Hayman, Barbara Caputo, and Jan-Olof Eklundh.
\newblock The {KTH-TIPS} database, 2004.

\bibitem{Gatys:TextureSynthCNN:NIPS2015}
Leon~A. Gatys, Alexander~S. Ecker, and Matthias Bethge.
\newblock Texture synthesis using convolutional neural networks.
\newblock In {\em Proceedings of the 28th International Conference on Neural
  Information Processing Systems - Volume 1}, page 262–270, Cambridge, MA,
  USA, 2015. MIT Press.

\bibitem{Gatys:StyleTransfer:CVPR2016}
Leon~A. Gatys, Alexander~S. Ecker, and Matthias Bethge.
\newblock Image style transfer using convolutional neural networks.
\newblock In {\em 2016 IEEE Conference on Computer Vision and Pattern
  Recognition (CVPR)}, pages 2414--2423, 2016.

\bibitem{Goodfellow:GAN:NIPS2014}
Ian~J. Goodfellow, Jean Pouget-Abadie, Mehdi Mirza, Bing Xu, David
  Warde-Farley, Sherjil Ozair, Aaron Courville, and Yoshua Bengio.
\newblock Generative adversarial nets.
\newblock In {\em Proceedings of the 27th International Conference on Neural
  Information Processing Systems - Volume 2}, page 2672–2680, Cambridge, MA,
  USA, 2014. MIT Press.

\bibitem{heeger95b}
David~J. Heeger and James~R. Bergen.
\newblock Pyramid-based texture analysis/synthesis.
\newblock In {\em Proc.~Int.~Conf.~Image Processing (ICIP), vol.~III}, pages
  648--651, Washington, DC, Oct. 1995.

\bibitem{Henzler:3DTexture:CVPR2020}
Philipp Henzler, Niloy~J. Mitra, and Tobias Ritschel.
\newblock Learning a neural {3D} texture space from {2D} exemplars.
\newblock In {\em 2020 IEEE/CVF Conference on Computer Vision and Pattern
  Recognition (CVPR)}, pages 8353--8361, 2020.

\bibitem{Ho:DDPM:NeurIPS2020}
Jonathan Ho, Ajay Jain, and Pieter Abbeel.
\newblock Denoising diffusion probabilistic models.
\newblock In {\em Advances in Neural Information Processing Systems},
  volume~33, pages 6840--6851. Curran Associates, Inc., 2020.

\bibitem{Isola:pix2pix:CVPR2017}
Phillip Isola, Jun-Yan Zhu, Tinghui Zhou, and Alexei~A. Efros.
\newblock Image-to-image translation with conditional adversarial networks.
\newblock In {\em 2017 IEEE Conference on Computer Vision and Pattern
  Recognition (CVPR)}, pages 5967--5976, 2017.

\bibitem{Jetchev:SpatialGAN:NIPS2016W}
Nikolay Jetchev, Urs Bergmann, and Roland Vollgraf.
\newblock Texture synthesis with spatial generative adversarial networks.
\newblock In {\em NeurIPS Workshop on Adversarial Training}, 2016.

\bibitem{Johnson:Perceptual:ECCV2016}
Justin Johnson, Alexandre Alahi, and Li Fei-Fei.
\newblock Perceptual losses for real-time style transfer and super-resolution.
\newblock In Bastian Leibe, Jiri Matas, Nicu Sebe, and Max Welling, editors,
  {\em Computer Vision -- ECCV 2016}, pages 694--711. Springer International
  Publishing, 2016.

\bibitem{Karras:Ada:NIPS2022}
Tero Karras, Miika Aittala, Janne Hellsten, Samuli Laine, Jaakko Lehtinen, and
  Timo Aila.
\newblock Training generative adversarial networks with limited data.
\newblock In {\em Advances in Neural Information Processing Systems},
  volume~33, pages 12104--12114. Curran Associates, Inc., 2020.

\bibitem{Karras:StyleGAN3:NIPS2021}
Tero Karras, Miika Aittala, Samuli Laine, Erik H\"{a}rk\"{o}nen, Janne
  Hellsten, Jaakko Lehtinen, and Timo Aila.
\newblock Alias-free generative adversarial networks.
\newblock In {\em Advances in Neural Information Processing Systems},
  volume~34, pages 852--863. Curran Associates, Inc., 2021.

\bibitem{Karras:StyleGAN:CVPR2020}
Tero Karras, Samuli Laine, and Timo Aila.
\newblock A style-based generator architecture for generative adversarial
  networks.
\newblock In {\em 2019 IEEE/CVF Conference on Computer Vision and Pattern
  Recognition (CVPR)}, pages 4396--4405, 2019.

\bibitem{Karras:StyleGAN2:CVPR2020}
Tero Karras, Samuli Laine, Miika Aittala, Janne Hellsten, Jaakko Lehtinen, and
  Timo Aila.
\newblock Analyzing and improving the image quality of {StyleGAN}.
\newblock In {\em 2020 IEEE/CVF Conference on Computer Vision and Pattern
  Recognition (CVPR)}, pages 8107--8116, 2020.

\bibitem{Kingma:Adam:ICLR2015}
Diederik~P. Kingma and Jimmy Ba.
\newblock Adam: {A} method for stochastic optimization.
\newblock In {\em 3rd International Conference on Learning Representations,
  {ICLR} 2015, San Diego, CA, USA, May 7-9, 2015, Conference Track
  Proceedings}, 2015.

\bibitem{Alex:CNN:ACM2017}
Alex Krizhevsky, Ilya Sutskever, and Geoffrey~E. Hinton.
\newblock Imagenet classification with deep convolutional neural networks.
\newblock {\em Commun. ACM}, 60(6):84–90, may 2017.

\bibitem{Ledig:SRGAN:CVPR2017}
Christian Ledig, Lucas Theis, Ferenc Huszár, Jose Caballero, Andrew
  Cunningham, Alejandro Acosta, Andrew Aitken, Alykhan Tejani, Johannes Totz,
  Zehan Wang, and Wenzhe Shi.
\newblock Photo-realistic single image super-resolution using a generative
  adversarial network.
\newblock In {\em 2017 IEEE Conference on Computer Vision and Pattern
  Recognition (CVPR)}, pages 105--114, 2017.

\bibitem{Levina:TextureSynMRF:AS06}
Elizaveta Levina and Peter~J. Bickel.
\newblock {Texture synthesis and nonparametric resampling of random fields}.
\newblock {\em The Annals of Statistics}, 34(4):1751 -- 1773, 2006.

\bibitem{Lin:texture:ICCV2015}
Tsung-Yu Lin, Aruni RoyChowdhury, and Subhransu Maji.
\newblock Bilinear cnn models for fine-grained visual recognition.
\newblock In {\em 2015 IEEE International Conference on Computer Vision
  (ICCV)}, pages 1449--1457, 2015.

\bibitem{Liu:bow2cnn:IJCV2019}
Li Liu, Jie Chen, Paul Fieguth, Guoying Zhao, Rama Chellappa, and Matti
  Pietik\"{a}inen.
\newblock From bow to cnn: Two decades of texture representation for texture
  classification.
\newblock {\em Int. J. Comput. Vision}, 127(1):74–109, jan 2019.

\bibitem{Oechsle:TexField:ICCV2019}
Michael Oechsle, Lars Mescheder, Michael Niemeyer, Thilo Strauss, and Andreas
  Geiger.
\newblock Texture fields: Learning texture representations in function space.
\newblock In {\em 2019 IEEE/CVF International Conference on Computer Vision
  (ICCV)}, pages 4530--4539, 2019.

\bibitem{Paget:TextureSynMRF:TIP98}
R. Paget and I.D. Longstaff.
\newblock Texture synthesis via a noncausal nonparametric multiscale markov
  random field.
\newblock {\em IEEE Transactions on Image Processing}, 7(6):925--931, 1998.

\bibitem{porat89}
M. Porat and Y.~Y. Zeevi.
\newblock Localized texture processing in vision: Analysis and synthesis in
  {G}aborian space.
\newblock {\em {IEEE} Trans. Biomed. Eng.}, 36(1):115--129, 1989.

\bibitem{Portenier:GramGAN:NIPS2020}
Tiziano Portenier, Siavash Arjomand~Bigdeli, and Orcun Goksel.
\newblock {GramGAN}: Deep {3D} texture synthesis from {2D} exemplars.
\newblock In H. Larochelle, M. Ranzato, R. Hadsell, M.F. Balcan, and H. Lin,
  editors, {\em Advances in Neural Information Processing Systems}, volume~33,
  pages 6994--7004. Curran Associates, Inc., 2020.

\bibitem{portilla00}
Javier Portilla and Eero~P. Simoncelli.
\newblock A parametric texture model based on joint statistics of complex
  wavelet coefficients.
\newblock {\em International Journal of Computer Vision}, 40(1):49--71, Oct.
  2000.

\bibitem{Richardson:pSp:CVPR2021}
Elad Richardson, Yuval Alaluf, Or Patashnik, Yotam Nitzan, Yaniv Azar, Stav
  Shapiro, and Daniel Cohen-Or.
\newblock Encoding in style: a {StyleGAN} encoder for image-to-image
  translation.
\newblock In {\em 2021 IEEE/CVF Conference on Computer Vision and Pattern
  Recognition (CVPR)}, pages 2287--2296, 2021.

\bibitem{Sharan:FMD:JOV2014}
Lavanya Sharan, Ruth Rosenholtz, and Edward~H. Adelson.
\newblock {Accuracy and speed of material categorization in real-world images}.
\newblock {\em Journal of Vision}, 14(9):12--12, 08 2014.

\bibitem{Shen:InterFaceGAN:CVPR2020}
Yujun Shen, Jinjin Gu, Xiaoou Tang, and Bolei Zhou.
\newblock Interpreting the latent space of {GANs} for semantic face editing.
\newblock In {\em 2020 IEEE/CVF Conference on Computer Vision and Pattern
  Recognition (CVPR)}, pages 9240--9249, 2020.

\bibitem{Simonyan:VGG:ICLR2015}
Karen Simonyan and Andrew Zisserman.
\newblock Very deep convolutional networks for large-scale image recognition.
\newblock In {\em 3rd International Conference on Learning Representations,
  {ICLR} 2015, San Diego, CA, USA, May 7-9, 2015, Conference Track
  Proceedings}, 2015.

\bibitem{Sitzmann:INR:NIPS2020}
Vincent Sitzmann, Julien Martel, Alexander Bergman, David Lindell, and Gordon
  Wetzstein.
\newblock Implicit neural representations with periodic activation functions.
\newblock In {\em Advances in Neural Information Processing Systems},
  volume~33, pages 7462--7473. Curran Associates, Inc., 2020.

\bibitem{sohl:diffusion:ICML2015}
Jascha Sohl-Dickstein, Eric~A. Weiss, Niru Maheswaranathan, and Surya Ganguli.
\newblock Deep unsupervised learning using nonequilibrium thermodynamics.
\newblock In {\em Proceedings of the 32nd International Conference on
  International Conference on Machine Learning - Volume 37}, page 2256–2265.
  JMLR.org, 2015.

\bibitem{Ulyanov:TexNet:ICML2016}
Dmitry Ulyanov, Vadim Lebedev, Andrea Vedaldi, and Victor~S. Lempitsky.
\newblock Texture networks: Feed-forward synthesis of textures and stylized
  images.
\newblock In {\em Proceedings of the 33nd International Conference on Machine
  Learning, {ICML} 2016, New York City, NY, USA, June 19-24, 2016}, volume~48
  of {\em {JMLR} Workshop and Conference Proceedings}, pages 1349--1357.
  JMLR.org, 2016.

\bibitem{Ulyanov:imp_texturenet:CVPR2017}
Dmitry Ulyanov, Andrea Vedaldi, and Victor Lempitsky.
\newblock Improved texture networks: Maximizing quality and diversity in
  feed-forward stylization and texture synthesis.
\newblock In {\em 2017 IEEE Conference on Computer Vision and Pattern
  Recognition (CVPR)}, pages 4105--4113, 2017.

\bibitem{wang04ssimOverview}
Zhou Wang, Alan~C. Bovik, Hamid~R. Sheikh, and Eero~P. Simoncelli.
\newblock Image quality assessment: From error visibility to structural
  similarity.
\newblock {\em {IEEE} Trans. Image Process.}, 13(4):600--612, Apr. 2004.

\bibitem{Xia:InvertSurvey:PAMI}
Weihao Xia, Yulun Zhang, Yujiu Yang, Jing-Hao Xue, Bolei Zhou, and Ming-Hsuan
  Yang.
\newblock {GAN} inversion: A survey.
\newblock {\em IEEE Transactions on Pattern Analysis and Machine Intelligence},
  pages 1--17, 2022.

\bibitem{Xu:GH-Feat:CVPR2021}
Yinghao Xu, Yujun Shen, Jiapeng Zhu, Ceyuan Yang, and Bolei Zhou.
\newblock Generative hierarchical features from synthesizing images.
\newblock In {\em 2021 IEEE/CVF Conference on Computer Vision and Pattern
  Recognition (CVPR)}, pages 4430--4430, 2021.

\bibitem{Xue:GTOS:PAMI2022}
Jia Xue, Hang Zhang, Ko Nishino, and Kristin~J. Dana.
\newblock Differential viewpoints for ground terrain material recognition.
\newblock {\em IEEE Transactions on Pattern Analysis and Machine Intelligence},
  44(3):1205--1218, 2022.

\bibitem{Zhang:LPIPS:CVPR2018}
Richard Zhang, Phillip Isola, Alexei~A. Efros, Eli Shechtman, and Oliver Wang.
\newblock The unreasonable effectiveness of deep features as a perceptual
  metric.
\newblock In {\em 2018 IEEE/CVF Conference on Computer Vision and Pattern
  Recognition}, pages 586--595, 2018.

\bibitem{Zhu:IDInvert:ECCV2020}
Jiapeng Zhu, Yujun Shen, Deli Zhao, and Bolei Zhou.
\newblock In-domain {GAN} inversion for real image editing.
\newblock In Andrea Vedaldi, Horst Bischof, Thomas Brox, and Jan-Michael Frahm,
  editors, {\em Computer Vision -- ECCV 2020}, pages 592--608. Springer
  International Publishing, 2020.

\bibitem{Zhu:ManipulateGAN:ECCV2016}
Jun-Yan Zhu, Philipp Kr{\"a}henb{\"u}hl, Eli Shechtman, and Alexei~A. Efros.
\newblock Generative visual manipulation on the natural image manifold.
\newblock In Bastian Leibe, Jiri Matas, Nicu Sebe, and Max Welling, editors,
  {\em Computer Vision -- ECCV 2016}, pages 597--613. Springer International
  Publishing, 2016.

\bibitem{zhu96}
S. Zhu, Y.~N. Wu, and D. Mumford.
\newblock Filters, random fields and maximum entropy ({FRAME}): Towards a
  unified theory for texture modeling.
\newblock In {\em IEEE Conf.~Computer Vision and pattern Recognition}, pages
  693--696, 1996.

\bibitem{zujovic_tip13}
Jana Zujovic, Thrasyvoulos~N. Pappas, and David~L. Neuhoff.
\newblock Structural texture similarity metrics for image analysis and
  retrieval.
\newblock {\em {IEEE} Trans. Image Process.}, 22(7):2545--2558, July 2013.

\end{thebibliography}
\end{document}